\documentclass[11pt]{article}

\usepackage[utf8]{inputenc} 
\usepackage[T1]{fontenc}    
\usepackage{hyperref}       
\usepackage{url}            
\usepackage{booktabs}       
\usepackage{nicefrac}       
\usepackage{microtype}      
\usepackage{xcolor}         
\usepackage[textsize = scriptsize]{todonotes}
\usepackage{amsmath,amssymb,amsfonts}
\usepackage{amsthm}
\usepackage{mathtools}
\mathtoolsset{showonlyrefs}
\usepackage{bm}
\usepackage{subcaption}
\usepackage{algorithm}
\usepackage{bbm}
\usepackage{algpseudocode}
\usepackage{comment}
\usepackage{tikz}
\usepackage{enumerate}

\usepackage{tikz}
\usepackage{pgfplots}
\usepackage[export]{adjustbox}
\usetikzlibrary{spy}

\usepackage{geometry}

\newtheorem{theorem}{Theorem}
 
\newtheorem{proposition}[theorem]{Proposition}
\newtheorem{remark}[theorem]{Remark}
\newtheorem{definition}[theorem]{Definition}

\newtheorem{corollary}[theorem]{Corollary}
\theoremstyle{definition}

\newcommand{\R}{\mathbb{R}}

\renewcommand{\d}{\mathrm{d}}
\renewcommand{\P}{\mathcal{P}}

\newcommand{\E}{\mathbb E}

\title{Y-Diagonal Couplings: Approximating Posteriors with Conditional Wasserstein Distances}
\author{Jannis Chemseddine, Paul Hagemann, Christian Wald}
\begin{document}

\maketitle
\begin{abstract}
  In inverse problems, many conditional generative models approximate the posterior measure by minimizing a distance between the joint measure and its learned approximation. While this approach also controls the distance between the posterior measures in the case of the Kullback--Leibler divergence, it does not hold true for the Wasserstein distance. We will introduce a conditional Wasserstein distance with a set of restricted couplings that equals the expected Wasserstein distance of the posteriors. By deriving its dual, we find a rigorous way to motivate the loss of conditional Wasserstein GANs. We outline conditions under which the vanilla and the conditional Wasserstein distance coincide. Furthermore, we will show numerical examples where training with the conditional Wasserstein distance yields favorable properties for posterior sampling.
\end{abstract}

\section{Introduction}
Many sampling algorithms for the posterior in Bayesian inverse problems perform learning on some joint measure. This means that given some observations $y$ with measure $P_Y$ one learns some probability measure $P_{Y,G_{\theta}}$, where $G_{\theta}$ also depends on $y$.
Many approaches minimize some loss of the form $$L(\theta) = D(P_{Y,X},P_{Y,G_{\theta}}),$$
where $D$ denotes a suitable distance on the space of joint measures. For instance this is done in the frame of conditional (stochastic) normalizing flows \cite{ardizzone2019guided,HHS22,HHG2023,winkler2019learning}, conditional GANs \cite{mirza2014conditional} or even for conditional gradient flows \cite{du2023nonparametric} for the Wasserstein metric.

A recent paper \cite{CondWasGen} investigated the relation between the joint measures $D(P_{Y,Z}, P_{Y,X})$ and its relation to the expected error between the posterior $E_{y \sim P_Y}\left[D(P_{Z|Y=y}, P_{X|Y=y}) \right].$ For the Kullback--Leibler (KL) divergence, it follows from the chain rule of the KL divergence \cite[Theorem 2.5.3]{cover_inf} that 
 $$E_{y \sim P_Y}\left[\mathrm{KL}(P_{X|Y=y}, P_{Z|Y=y}) \right] = \mathrm{KL}(P_{Y,X}, P_{Y,Z}).$$  

Such results are important as they show that it is possible to approximate the posterior by approximating the joint distribution. The paper \cite{altekrueger2023conditional} shows robustness of such conditional generative models under the assumption that the expected error to the posterior $E_{y \sim P_Y}\left[W_1(P_{X|Y=y}, P_{Z|Y=y}) \right]$ is small.

In \cite[Theorem 2]{CondWasGen} it is claimed that for $D = W_1$ (the Wasserstein-1 metric \cite{villani2009optimal}) it holds that  

$$W_1(P_{Y,X}, P_{Y,Z}) = \mathbb{E}_{y\sim P_Y}\left[W_1(P_{X|Y=y}, P_{Z|Y=y})\right].$$

In general we provide a counterexample to this claim. Intuitively, the issue arises when the optimal transport plan needs to transport mass in the $Y$-component. 
This is the motivation for only considering plans that do not have mass transport in the $Y$-component, which is also what the paper \cite{CondWasGen} uses in their proof. It leads to a definition of the conditional Wasserstein distance which we will denote by $W_{p,Y}$. The main change is to restrict the set of transport plans  to $\Gamma_Y=\Gamma_Y(P_{Y,X},P_{Y,Z})$, the set of plans $\alpha$ s.t. $(\pi_{1,3})_{\sharp}\alpha = \Delta_{\sharp}P_{Y}$ where $\Delta(y) = (y,y)$ is the diagonal map. Then we define the conditional Wasserstein distance $W_{p,Y}^p(P_{Y,X},P_{Y,Z})$ as
 \begin{align*}
     \underset{\alpha\in \Gamma_Y}{\mathsf{inf}}\int\|(y_1,x_1)-(y_2,x_2)\|^p\d \alpha
 \end{align*}

Inspired by \cite{CondWasGen}, we show that this conditional Wasserstein distance indeed corresponds to an expectation over the posteriors, i.e. $$W_{p,Y}^p(P_{Y,X},P_{Y,Z}) = \mathbb{E}_{y\sim P_Y}\left[W_p^p(P_{X|Y=y}, P_{Z|Y=y})\right].$$

Then, deriving the dual formulation as in the theory of optimal transport \cite{villani2009optimal}, we can relate our $W_{1,Y}$ to the loss used in \cite{adler2018deep, martin2021exchanging} under milder conditions and with a slightly different class of discriminators. This yields further insights on conditional Wasserstein GANs (WGAN) \cite{adler2018deep,CondWasGen, liu2021wasserstein, ZHENG20201009}, where the loss function was heuristically motivated. In particular we are able to prove more regularity of the dual functions than in \cite{adler2018deep, martin2021exchanging}. Furthermore, we will find suitable conditions under which the conditional Wasserstein distance $W_{p,Y}$ and $W_p$ are close. This also relates to Wasserstein flows which target the joint distribution \cite[Section 4.2]{du2023nonparametric}. Our approach will be supported by numerical experiments validating our theoretical results. 
\paragraph{Contributions}
\begin{itemize}
\item We introduce the conditional Wasserstein distance and highlight its relevance to conditional generative models for inverse problems. 
\item We calculate its dual and recover standard WGAN loss formulations \cite{adler2018deep,CondWasGen,  martin2021exchanging} with more regularity than in previous literature.
\item We identify situations in which it is sufficient to learn with the vanilla Wasserstein distance and still guarantee posterior recovery. 
\item We show how to train (approximately) with the conditional Wasserstein distance \emph{without} invoking the dual formulation,  using differentiable OT algorithms \cite{feydy2019interpolating}.
\end{itemize}
\paragraph{Related work}
Our work operates in the intersection of conditional generative modelling \cite{adler2018deep, ardizzone2019guided, mirza2014conditional} and (computational) optimal transport \cite{MAL-073, villani2009optimal}. Most related to our work are the conditional Wasserstein GAN papers, which do not agree on whether to impose the Lipschitz condition on both components or only with respect to $x$ \cite{adler2018deep, CondWasGen, ray2023solution}. In the GAN literature for example \cite{shahbazi2022collapse}, observe a lack of diversity in samples generated by conditional GANs, which is dubbed conditioning collapse. 
The recent work \cite[Theorem 2]{kim2023wasserstein} derive an inequality based on restricting the admissible couplings in the their optimal transport formulation to so called conditional sub-couplings. However, their paper is focused on learning geodesics whereas our paper is focused on understanding the theoretical foundations of these metrics. 

In optimal transport literature there has been a stream of class conditional optimal transport distances used in domain adaption \cite{nguyen2022cycle,rakotomamonjy2021optimal}. In particular conditional OT as in \cite{tabak2021data} is relevant as they consider optimal transport plans for each condition $y$ minimizing $\mathbb{E}_y[W_1(P_{X|Y=y}, G(\cdot,y)_{\#}P_Z)]$. However they relax their problem using a KL divergence.
Similarly the field of gradient flows \cite{ambrosio2021lectures, gig} investigate the same object which shows up as the tangent space of the the Wasserstein space. In \cite[Remark 7]{hagemann2023posterior} an inequality between the joint Wasserstein and the expected value over the conditionals is derived , crucially the result requires compactly supported measures and regularity of the associated posterior densities.

In \cite{bunne2022supervised} the supervised training of conditional Monge maps is proposed, for which they solve the dual using convex neural networks. As a follow-up \cite{manupriya2023empirical} proposed a relaxation which only needs the samples from the joint distribution involving the MMD. 

Apparently similar objects are treated frequently in the literature. Therefore, we want to contribute to the theoretical foundations of these approaches by introducing the conditional Wasserstein distance. 

\section{Background}
The motivation for our investigation comes from Bayesian inverse problems of the form $Y = f(X) + \xi$, where $f$ denotes the forward operator, $\xi$ an appropriate noise model, $Y$ the random vector of observations and $X$ the hidden random vector of parameters. We will denote their measures by $P_Y$ and $P_X$ respectively. The joint measure will be denoted by $P_{Y,X}$ and the posterior measure given $y$ by $P_{X|Y=y}$ which is the disintegration of $P_{Y,X}$ by $P_Y$ at point $y$. For more information on Bayesian inverse problems we refer to \cite{stuart_2010}. 
For a two measurable spaces $X,Y$, a measure $\mu$ on $X$ and a measurable function $f:X\to Y$ we denote the push forward measure on $Y$ by $f_{\sharp}\mu$. For a product space $\prod_{i}X_i$ we denote the projection onto the $i_1,\ldots,i_n$-th component by $\pi^{i_1,\ldots,i_n}$. 
Let $X \subseteq \mathbb{R}^d$ with the Borel $\Sigma$ algebra and let $\mathcal{P}_p(X)$ be the set of all probability measures on $X$ with finite $p$-th moments. 
The metric on $\mathcal{P}_p(X)$ will be the Wasserstein-p metric \cite{villani2009optimal}, which is given by $W_p(\mu, \nu) = \left(\inf_{\gamma \in \Gamma} \int \Vert x-y \Vert^p d\gamma(x,y)\right)^{\frac 1 p}$ for $p\in [1,\infty)$. Here $\Gamma=\Gamma(\mu,\nu)$ denotes the set of all couplings, i.e. all probability measures 
$\gamma\in \P(X\times X)$ with marginals $\pi^1_{\sharp}\gamma=\mu$ and $\pi^2_{\sharp}\gamma =\nu$. For two empirical measures $\mu=\frac 1 n\sum_{i=1}^n\delta_{a_i}$ and $\nu=\frac 1 n \sum_{i=1}^n\delta_{b_i}$ with the same number of particles the Wasserstein distance can also be written \cite[Proposition 2.1]{MAL-073} as
\begin{align}\label{eq:assign}
W_p^p = \inf_{\pi\in S_n}\frac 1 n \sum_{i=1}^n\Vert a_i-b_{\pi(i)}\Vert^p.
\end{align}
Here $S_n$ is the set of permutations and $\pi\in S_n$ coresponds to the coupling $\frac 1 n \sum_i \delta_{(a_i,b_{\pi(i)})}$.

\section{A Simple Example}
We will first provide a simple example showing that one cannot expect equality between $W_1(P_{Y,X}, P_{Y,Z})$ and $\mathbb{E}_y(W_1[P_{X|Y=y}, P_{Z|Y=y})]$ to hold true. This will be enlightening as it invalidates a claim in \cite{CondWasGen} and motivates our interest in coupling plans that are diagonal in the first variable.

First we fix a probability space $(\Omega, \mathcal{F}, P)$ with $\Omega = (\omega_1, \omega_2)$, $\mathcal{F} = 2^{\Omega}$ and the uniform distribution $P(\omega_1) = P(\omega_2) = \frac{1}{2}$.
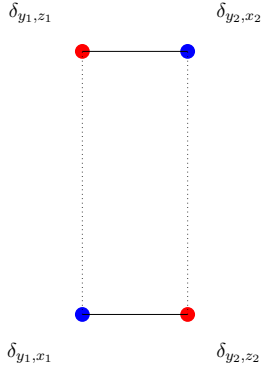
\begin{figure}
\centering
\scalebox{0.7}{
\begin{tikzpicture}
    \fill [blue] (0,0) circle [radius=4pt];
    \node[below left=14pt of {(0,0)}, outer sep=2pt,fill=white] {$\delta_{y_1,x_1}$};
    \fill [red] (0,5) circle [radius=4pt];
    \node[above left=14pt of {(0,5)}, outer sep=2pt,fill=white] {$\delta_{y_1,z_1}$};
    \fill [red] (2,0) circle [radius=4pt];
    \node[below right=14pt of {(2,0)}, outer sep=2pt,fill=white] {$\delta_{y_2,z_2}$};
    \fill [blue] (2,5) circle [radius=4pt];
    \node[above right=14pt of {(2,5)}, outer sep=2pt,fill=white] {$\delta_{y_2,x_2}$};
    \draw [black] (0,0) -- (2,0);
    \draw [black] (0,5) -- (2,5);

    \draw [dotted, black] (0,0) -- (0,5);
    \draw [dotted, black] (2,0) -- (2,5);
\end{tikzpicture}
}

    \caption{Visualization of the example. The diagonal coupling is visualized via dotted lines, where the blue dots belong to $P_{Y,X}$ and the red dots correspond to $P_{Y,Z}$ . The optimal non-diagonal coupling is the solid line.}
    \label{fig:enter-label}
\end{figure}
Let $\|\cdot\|$ denote the 2 norm on $\R^2$. We define the following random variables for $n>0$
\begin{center}
\begin{tabular}{c| c c }
   & $\omega_1$ & $\omega_2$ \\ 
\hline
 X & 0 & n \\  
 Y & 0 & 1 \\   
Z & n& $0$    
\end{tabular}
\end{center}
Then we have 
\begin{align*}
P_{Y,X} = \frac 1 2 \delta_{1,n} + \frac 1 2\delta_{0,0} && P_{Y,Z}= \frac 1 2 \delta_{1,0} + \frac 1 2 \delta_{0,n}
\end{align*}
which implies
\begin{align*}
&W_1(P_{Y,X},P_{Y,Z})  \\
&\overset{\eqref{eq:assign}}{=} \frac{1}{2}\min\big\{\|(0,0)-(1,0)\| + \|(1,n)-(0,n)\|,\\
&\qquad\qquad \|(0,0)-(0,n)\|+\|(1,n)-(1,0)\|\big\}\\
&=\frac 1 2 + \frac 1 2 = 1
\end{align*}
Furthermore
\begin{align*}
&P_{X|Y=0}= \delta_{0}&&P_{X|Y=1}=\delta_{n}\\
&P_{Z|Y=0} = \delta_{n}&&P_{Z|Y=1} = \delta_0\\
&P_{Y} = \frac 1 2 \delta_0 + \frac 1 2 \delta_1
\end{align*}
and thus
\begin{align*}
\mathbb{E}_y[W_1(P_{X|Y=y},P_{Z|Y=y})] = \frac 1 2 n + \frac 1 2 n =n
\end{align*}
\color{black}

Hence we obtain for all $n>0$ that 
\begin{align}
W_1(P_{Y,X}, P_{Y,Z}) =\frac 1 n \mathbb{E}_y\left[W_1(P_{X|Y=y}, P_{Z|Y=y})\right].
\end{align}
Note that if we forbid the coupling to move mass across the $y$ direction, we actually would obtain equality, which motivates our definition of conditional Wasserstein distance. 

\begin{remark}
In the article \cite{CondWasGen}, they consider a summation metric on the $(X,Y)$-space namely $\Vert (x_1,y_1)- (x_2,y_2) \Vert_{sum} = \Vert x_1-x_2 \Vert + \Vert y_1-y_2 \Vert.$ However since for the elements in the norms of the summation metric and the 2-norm coincide our counterexample is still valid in this case.
\end{remark}
\section{Conditional Wasserstein Distance}
In general one can only expect the inequality
\[
W_1(P_{Y,Z},P_{Y,X}) \leq \mathbb{E}_y[W_1(P_{X|Y=y},P_{Z|Y=y})]
\] 
to hold true. This makes intuitive sense since by gluing the couplings from the conditionals together yields a coupling for the joint distribution. 

The main idea to obtain the other inequality for the conditional Wasserstein distance is to allow only couplings (or Monge maps) which leave the $Y$-component invariant. In terms of Monge maps this means that we are considering functions $(\mathrm{Id}, T(y,\cdot)): (y,x)\mapsto (y,T(y,x))$ where $T:\R^n\times \R^d\to\R$ and
$(\mathrm{Id}, T(y,\cdot))_{\#}P_{Y,X} = P_{Y,Z}$. The general case is handled in the follwing definition.

\begin{definition}{Conditional Wasserstein Distance.}
        Let $X,Z:\Omega\to B\subset\R^{d}, Y:\Omega\to A\subset\R^{n}$ be random variables of finite $p$-th moment. Let $\Gamma_Y(P_{Y,X},P_{Y,Z})$ be the set of plans $\alpha\in\Gamma(P_{Y,X}, P_{Y,Z})$ s.t. $(\pi_{1,3})_{\sharp}\alpha = \Delta_{\sharp}P_{Y}$ where $\Delta:A\to A^2$ is the diagonal map. We then define 
        \begin{align*}   
        &W_{p,Y}(P_{Y,X},P_{Y,Z}) = \\ &\left( \underset{\alpha\in \Gamma_Y(P_{Y,X}, P_{Y,Z})}{\mathsf{inf}}\int_{(A\times B)^2}\|(y_1,x_1)-(y_2,x_2)\|^p\d \alpha\right)^{\frac 1 p}
        \end{align*}
\end{definition}

\begin{proposition}\label{cond:plan}
     
    Let $X,Y,Z$ be as in the previous definition. Then
        \begin{align*}
               W_{p,Y}(P_{Y,X},P_{Y,Z})^p= \underset{y\sim P_Y}{\mathbb{E}}\big[ W_p^p(P_{X|Y=y},P_{Z|Y=y})\big]
        \end{align*}
\end{proposition}
\begin{proof}
    First we show $\geq$. 
        Let $\gamma_{b_1,b_2}$ be the disintegration of $\alpha\in\Gamma_Y(P_{Y,X},P_{Y,Z})$ w.r. to $\pi^{1,3}_{\sharp}\alpha$. Then
        \begin{align}\label{mod:was}
                &\int_{(A\times B)^2}\|(y_1,x_1)-(y_2,x_2)\|^p\d \alpha  \\
                &= \int_{(A\times B)^2} \|(y_1,x_1)-(y_2,x_2)\|^p \d \gamma_{y_1,y_2}(x_1,x_2)\d (\pi_{1,3})_{\sharp}\alpha\\
                &= \int_{(A\times B)^2}\|(y_1,x_1)-(y_2,x_2)\|^p\d \gamma_{y_1,y_2}(x_1,x_2)\d \Delta_{\sharp}P_Y\\
                &= \int_{A\times B^2}\|(y,x_1)-(y,x_2)\|^p\d\gamma_{y,y}(x_1,x_2)\d P_Y(y)\\
                &= \int_{A\times B^2}\|x_1-x_2\|^p \d\gamma_{y,y}(x_1,x_2)\d P_Y(y).
        \end{align}
     
        Thus it is enough to show $\gamma_{y,y} \in \Gamma(P_{X|Y=y},P_{Z|Y=y})$ a.e. which means $(\pi_1)_{\sharp}\gamma_{y,y}=P_{X|Y=y}$ a.e. and similarly for $Z|Y$. Since $P_{X|Y=y}$ is defined via
        \[
            \int_{A\times B}f(y,x)\d P_{X|Y=y}(x)\d P_Y(y)=\int_{A\times B}f(y,x)\d P_{Y,X}
            \]
            for all Borel measurable functions $f:A\times B\to [0,\infty)$, this follows from
        \begin{align*}
                &\int_{A\times B} f(y,x_1)\d(\pi_1)_{\sharp}(\gamma_{y,y})(x_1)\d P_Y(y) \\ &=\int_{A^2\times B}f(y_1,x_1) \d(\pi_1)_{\sharp}(\gamma_{y_1,y_2})(x_1)\d (\Delta)_{\sharp}P_Y(y_1,y_2)\\
                &=\int_{A^2\times B}f(y_1,x_1) \d(\pi_1)_{\sharp}(\gamma_{y_1,y_2})(x_1)\d (\pi_{1,3})_{\sharp}\alpha(y_1,y_2)\\
                &=\int_{A^2\times B^2}f(y_1,x_1) \d\gamma_{y_1,y_2}(x_1,x_2)\d (\pi_{1,3})_{\sharp}\alpha(y_1,y_2)\\
                &= \int_{A^2 \times B^2}f(y_1,x_1)\d\alpha\\
                &= \int_{A\times B} f(y_1,x_1)(\pi_{1,2})_{\sharp}\alpha\\
                &= \int_{A\times B} f(y_1,x_1)\d P_{Y,X}(y_1,x_1).
        \end{align*}
        Now we show $\leq$. For $y\in A$ let $\gamma_y\in \Gamma(P_{X|Y=y},P_{Z|Y=y})$ be an optimal plan i.e. 
        \[
        \mathcal{W}_p^p(P_{X|Y=y},P_{Z|Y=y})= \int_{A^2} \|x_1-x_2\|^p\d \gamma_y 
         \]
         Let $\alpha = \int_{A}\d \delta_{y_1}(y_2)\d \gamma_{y_1}(x_1,x_2)\d P_Y(y_1)$. Then as in the proof of \cite[Theorem 2]{CondWasGen} 
        \begin{align*}
            &\int_{A}\mathcal{W}_p^p(P_{X|Y=y},P_{Z|Y=y})\d P_Y(y)\\&= \int_{A\times B^2} \|x_1-x_2\|^p\d \gamma_y(x_1,x_2)\d P_Y(y) \\
            &= \int_{(A\times B)^2}\|(y_1,x_1) - (y_2,x_2)\|^p\d \alpha.
        \end{align*}
        Thus it suffices to show that $\alpha \in\Gamma_Y(P_{Y,X},P_{Y,Z})$ which means $\pi^{1,3}_{\sharp}\alpha= \Delta_{\sharp}P_Y$, $\pi^{1,2}_{\sharp}\alpha=P_{Y,X}$ and $\pi^{3,4}_{\sharp}\alpha = P_{Y,Z}$. The first equality follows from
        \begin{align*}
        &\int_{A^2}f(y_1,y_2)\d \pi^{1,3}_{\sharp}\alpha \\&= \int_{(A\times B)^2}f\circ \pi^{1,3}(y_1,x_1,y_2,x_2)\d\alpha \\
                &= \int_{(A\times B)^2}f(y_1,y_2)\d \delta_{y_1}(y_2)\d \gamma_{y_1}(x_1,x_2)\d P_Y(y_1)\\
                &= \int_{A}f(y,y)\d P_Y(y) \\
                &= \int_{A^2}f(y_1,y_2)\d(\Delta_{\sharp}P_Y)(y_1,y_2)
        \end{align*}
        for all test functions $f$. The second follows from
        \begin{align}
            &\int_{A\times B}f(y,x)\d\pi^{1,2}_{\sharp}\alpha \\&= \int_{(A\times B)^2}f(y_1,x_1)\d \delta_{y_1}(y_2)\d \gamma_{y_1}(x_1,x_2)\d P_Y(y_1)\\
            &=\int_{A\times B}f(y,x)\d\pi^1_{\sharp}\gamma_y(x)\d P_Y(y)\\
            &=\int_{A\times B}f(y,x)\d P_{X|Y=y}(x)\d P_Y(y)\\
            &=\int_{A \times B}f(y,x) \d P_{Y,X}(y,x)
        \end{align}
        for all test functions $f$. The third equality follows analogously.
\end{proof}
\begin{corollary}
\label{cor_plan}
    In the definition of $W_{p,Y}$ the infimum is attained and for optimal plans $\gamma_y\in\Gamma(P_{X|Y=y},P_{Z,Y=y})$ the plan $\alpha=\int_{A}\d \delta_{y_1}(y_2)\d \gamma_{y_1}(x_1,x_2)\d P_Y(y_1)$ is optimal.
\end{corollary}
\begin{remark}
In the PhD thesis \cite[equation 4.7]{gig} a very similar object as our $\mathbb{E}_y\left[W_1(P_{X|Y=y}, P_{Z|Y=y})\right]$ is treated. They are interested in the Wasserstein tangent space at some measure $\mu$ on $\R^d$ which is $P_Y$ in our case. They then investigate the space $\mathcal{P}_2(\R^d\times \R^d)_{P_Y}$ which is the set of all measures $\alpha$ of finite second moment and $\pi^1_{\sharp}\alpha=P_Y$. In particular $P_{Y,X}, P_{Y,Z}$ belong to $\mathcal{P}_2(\R^d \times \R^d)_{P_Y}$. They then define a distance on $\mathcal{P}_2(\R^d\times \R^d)_{P_Y}$ by
\[
W_{P_Y}^2(\gamma_1,\gamma_2)=\int_{\R^d} W^2((\gamma_1)_y,(\gamma_2)_y)\d P_Y(y)
\]
for $(\gamma_i)_x$ the disintegration w.r. to $P_Y$. This is exactly the right hand side of proposition \ref{cond:plan}. In order to show that this is a metric they show an alternative description of $W_{P_Y}^2$ which is
\[
W_{P_Y}(\gamma_1,\gamma_2) = \underset{\alpha\in\mathcal{ADM}_{P_Y}(\gamma_1,\gamma_2)}{inf} \ \left(\int_{A\times B^2}\|x_1 -x_2\|^2\d \alpha\right)^{\frac 1 2}
\]where $\mathcal{ADM}_{P_Y}(\gamma_1,\gamma_2)=\{\alpha\in\mathcal{P}_2((\R^d)^3 ):\pi^{1,2}_{\sharp}\alpha=\gamma_1,\pi^{1,3}_{\sharp}\alpha=\gamma_2\}$. This is closely related to the left hand side of Proposition \ref{cond:plan} as we will see in the next proposition. Note that despite \cite{gig} only considers the case $p=2$ their arguments hold true for general $p\in[1,\infty)$. 
\end{remark}

\begin{proposition}
\label{prop:3plans}
    Consider random variables $X,Z :\Omega\to B$, $Y:\Omega\to A$ with finite $p$-th moment and let 
    \begin{align*}
    &\mathcal{ADM}_{P_Y}(\gamma_1,\gamma_2)
    =\{\alpha\in\mathcal{P}_p((A\times B^2):\pi^{1,2}_{\sharp}\alpha=\gamma_1,\pi^{1,3}_{\sharp}\alpha=\gamma_2\}.
    \end{align*}
    
    Then the map \[
    \pi^{1,2,4}_{\sharp}: \Gamma_Y(P_{Y,X},P_{Y,Z})\to \mathcal{ADM}_{P_Y}(P_{Y,X},P_{Y,Z})\]
    is a bijection and for $\alpha\in \Gamma_Y(P_{Y,X},P_{Y,Z})$ it holds that
    \begin{align*}
         &\int_{(A\times B)^2}\|(y_1,x_1)-(y_2,x_2)\|^p\d \alpha=\int_{A\times B^2}\|x_1-x_2\|^p\d \pi^{1,2,4}_{\sharp}\alpha.    
    \end{align*}
    In particular we have that $W_{P_Y} = W_{2,Y}$ for $p=2$.
\end{proposition}

\begin{corollary}
Let $\P_{p,Y}(A\times B)$ be the set of probability measures $\alpha\in \P_p(A\times B)$ with $\pi^1_{\sharp}\alpha=P_Y$. Then $W_{p,Y}$ is a metric on $\P_{p,Y}(A\times B)$.
\end{corollary}
\begin{proof}
It is shown in \cite[Theorem 4.4]{gig} that $W_{2,Y}$ is a metric. Since their proof also works for $p\in[1,\infty)$ the previous proposition yields the claim.
\end{proof}

\section{Dual Representation}
In this section we present a dual formulation of the conditional Wasserstein distance $W_{1,Y}$, similar in spirit to the usual dual formulations for the Wasserstein distance \cite{villani2009optimal}. The proof uses similar arguments as the short notes \cite{thickstun} and \cite{Basso2015AH}. In we will show that under the assumption of compact supports we have 
$W_{1,Y}(P_{Y,X},P_{Y,Z}) =  \sup_{h \in \tilde{D}} \mathbb{E}_{Y,X}[h] - \mathbb{E}_{Y,Z}[h]$
where $\tilde{D}$ denotes the set of upper semi-continuous functions $h = h(y,x)$ which are 1-Lipschitz with respect to $x$. 
 
In particular this closes a gap in the dual formulation considered in \cite{adler2018deep}, which was also discussed in \cite{martin2021exchanging}, where they fixed this gap under stronger assumptions. The dual formulation we derive gives more regularity of the dual functions. 

\begin{proposition}
     Let $X,Z:\Omega\to B\subset\R^d$ and $Y:\Omega\to A\subset\R^n$ be random variables and assume that $A,B$ are compact. Then we have the following dual representation
     \begin{align*}
         W_{1,Y}(P_{Y,X},P_{Y,Z}) =  \sup_{h \in \tilde{D}}\left\{ \mathbb{E}_{Y,X}[h] - \mathbb{E}_{Y,Z}[h]\right\}
     \end{align*}
     where $\tilde{D}$ denotes the set of bounded upper semi-continuous functions $h=h(y,x):A\times B\to \R$ satisfying $|h(y,x_1)-h(y,x_2)|\leq \|x_1-x_2\|$
     for all $y \in A,\ x_1,x_2\in B$.
\end{proposition}
\begin{proof}
Denote by $C_b = C_b(A\times B)$ the space of continuous bounded functions on $A \times B$ and by $M$ the set of nonnegative finite Borel measures $\alpha$ on $A\times B$ which are supported  at most on the diagonal. Formally this means that there exists a nonnegative finite Borel measure $\beta$ on $B$ s.t. $\pi^{1,3}_{\sharp}\alpha = (\Delta_{B})_{\sharp}\beta$. Note that
\[
\pi^1_{\sharp}\alpha = \pi^{1}_{\sharp}\pi^{1,3}_{\sharp}\alpha = \pi^{1}_{\sharp}(\Delta_B)_{\sharp}\beta = (id_{B})_{\sharp}\beta=\beta 
\]
and thus for $\alpha\in\Gamma(P_{Y,X},P_{Y,Z})$ the condition $\pi^{1,3}_{\sharp}\alpha = (\Delta_{B})_{\sharp}\beta$ is equivalent to $\pi^{1,3}_{\sharp}\alpha = (\Delta_{B})_{\sharp}P_Y$.\\
A standard tool for enforcing $\alpha\in \Gamma(P_{Y,X},P_{Y,Z})$ , see e.g. \cite[Section 1.2]{santambrogio2015optimal} , is to consider
\begin{align*}
&\sup_{f,g\in C_b(A\times B)}\E_{P_{Y,X}}[f] + \E_{P_{Y,Z}}[g] - \int (f+g)\d\alpha\\
&=\begin{cases}
    0 \quad\text{ if } \alpha\in\Gamma(P_{Y,X},P_{Y,Z})\\
    \infty \quad\text{else.}
\end{cases}\end{align*}

Thus for the Lagrangian

\begin{align*}
    &L(\alpha,f,g):= \mathbb{E}_{P_{Y,X}}[f] + \mathbb{E}_{P_{Y,Z}}[g]  
    + \underset{(A\times B)^2}{\int} \left(\Vert(y_1,x_1)-(y_2,x_2)\Vert -f(y_1,x_1) -g(y_2,x_2)\right)  \d \alpha 
\end{align*}
we have that
\begin{align*}
 W_{1,Y}(P_{Y,X},P_{Y,Z}) &= \inf_{\alpha\in \Gamma_Y}\int\|(y_1,x_1)-(y_2,x_2)\|\d\alpha \\
 &=\inf_{\alpha \in M} \sup_{f,g \in C_b} L(\alpha, f,g).
 \end{align*}

We show in Appendix \ref{app:ex} that strong duality holds in our case. Thus we may exchange infimum and supremum and get that 
$$W_{1,Y}(P_{Y,X},P_{Y,Z}) = \sup_{f,g \in C_b} \inf_{\alpha \in M}  L(\alpha, f,g).$$

From this we can infer that 
\begin{align}\label{eq:leq}
f(y,x_1) + g(y,x_2) \leq \Vert x_1 - x_2 \Vert
\end{align}
for all $y\in A$ as otherwise the attained infimum is $-\infty.$ Thus we can assume that $L(\alpha,f,g)\geq \mathbb{E}_{P_{Y,X}}[f] + \mathbb{E}_{P_{Y,Z}}[g] $ and choosing the plan $\alpha = 0 \in M$, we obtain that
$$\inf_{\alpha\in M}L(\alpha,f,g) = \mathbb{E}_{P_{Y,X}}[f] + \mathbb{E}_{P_{Y,Z}}[g]$$ for all $(f,g)\in D$ with
\[
D:= \{(f,g) \in C_b^2: f(y,x_1) + g(y,x_2) \leq \Vert x_2 - x_2 \Vert\}.
\]
Then it follows that 
\begin{align}\label{eq:wsup}
W_{1,Y}(P_{Y,X},P_{Y,Z}) = \sup_{(f,g) \in D} \mathbb{E}_{Y,X}[f] + \mathbb{E}_{Y,Z}[g].
\end{align}

For $(f,g)\in D$ define $\tilde{f}(y,x) = \inf_u \Vert x-u \Vert - g(y,u)$ . Then 
\begin{align}
    \tilde{f}(y,x) &= \inf_u \{\|x-u\| - g(y,u)\} \\
    &\leq \inf_u \{\|x-z\| + \|z-u\| - g(y,u)\} \\
    &= \tilde{f}(y,z) + \|x-z\|
\end{align}
    shows the 1-Lipschitz continuity of $\tilde{f}$ with respect to the second component. Using \eqref{eq:leq} we obtain that $\tilde{f}(y,x) \geq f(y,x)$. Since $\tilde{f}(y,x) \leq \|x-x\| - g(y,x)$ by definition we conclude that
    \begin{align}\label{eq:ftild}
    f(y,x) \leq \tilde{f}(y,x) \leq -g(y,x).
    \end{align}

By \ref{eq:ftild} we see that $\tilde{f}$ is bounded. Additionally $\tilde{f}$ also possesses some regularity, namely as an infimum over continuous functions it is upper semicontinuous in $(y,x)$ and 1-Lipschitz with respect to the second component i.e. $\tilde{f}\in\tilde{D}$. 

We thus conlude
\begin{align*}
&W_{1,Y}(P_{Y,X},P_{Y,Z})\\
&\overset{{\eqref{eq:wsup}}}{=} \sup_{(f,g) \in D} \mathbb{E}_{Y,X}[f] + \mathbb{E}_{Y,Z}[g] \\
&\overset{\leq}{\eqref{eq:ftild}} \sup_{\tilde{f} = \tilde{f}(f,g), (f,g) \in D} \mathbb{E}_{Y,X}[\tilde{f}] - \mathbb{E}_{Y,Z}[\tilde{f}]\\
&\leq \sup_{h \in \tilde{D}} \mathbb{E}_{Y,X}[h] - \mathbb{E}_{Y,Z}[h] \\
&= \sup_{h \in \tilde{D}} \inf_{\alpha \in \Gamma_Y}\int_{(A\times B)^2}h(y_1,x_1)-h(y_2,x_2) \d \alpha\\
&= \sup_{h\in \tilde{D}}\inf_{\alpha \in \Gamma_Y}\int_{(A\times B)^2}h(y_1,x_1)-h(y_1,x_2)\d \alpha \\
&\leq \inf_{\alpha \in \Gamma_Y} \int \Vert x_1 - x_2 \Vert \d\alpha \\
& = \inf_{\alpha \in \Gamma_Y}\int \Vert (y_1,x_1)-(y_2,x_2)\Vert \d\alpha \\
&= W_{1,Y}(P_{Y,X},P_{Y,Z}).
\end{align*}
This finishes the proof that 

$$W_{1,Y}(P_{Y,X},P_{Y,Z}) = \sup_{h\in \tilde{D}} \mathbb{E}_{Y,X}[h] - \mathbb{E}_{Y,Z}[h],$$
which is the objective often used in conditional Wasserstein GAN training \cite{adler2018deep, CondWasGen}.
\end{proof}

\section{Relation between Wasserstein Distances}
In some settings the conditional Wasserstein distance is equal or close to the Wasserstein distance of the joint distribution, i.e., in this cases the optimal transport plan is (almost) contained in the restricted set. In a first example we will look at the case when $Y,X$ resp. $Y,Z$ are independent, where we will show that the optimal coupling is just the product of an optimal coupling between $X$ and $Z$ and the identity coupling of $Y$. This is formalized in the next Proposition and also shown in \cite[Proposition 14]{kim2023wasserstein}.

\begin{proposition}
\label{prop:ind}
    Let $X,Z:\Omega\to B$ and $Y:\Omega\to A$ be random variables of finite $p$-th moment,  such that $Y,X$ are independent and $Y,Z$ are independent. Then
    \[ W_p(P_{Y,X},P_{Y,Z})=W_{p,Y}(P_{Y,X},P_{Y,Z}) = W_p(P_X,P_Z)
    \]
\end{proposition}

One approach closely related to so-called conditional Wasserstein flows \cite{du2023nonparametric} consists in obtaining data samples $(y_i,x_i) \sim P_{Y,X}$ and starting at points $(y_i, z_i)$ for some random $z_i \sim P_Z$ independent of $Y$ and $X$ and then construct a flow where the samples $(y_i,z_i)$ are transported to $(y_i,x_i)$ to obtain new samples. In the next proposition we argue that in this setting the optimal coupling is "expected" to be diagonal along $Y$ for randomly drawn samples.
Recall that for empirical measures $\mu=\frac{1}{k}\sum_{i=1}^k \delta_{a_i},\nu=\frac 1 k\sum_{i=1}^k \delta_{b_i}$ a transport plan $\pi\in \Gamma(\mu,\nu)$ can be described as a matrix $\pi\in \R^{k\times k}$ such that the row and column sums are $1$. Then the quadratic cost $c$ is 
\[
c(\pi,\mu,\nu)=\sum_{i,j}\Vert a_i-b_j\Vert^2\pi_{i,j}
\] and $W_2^2(\mu,\nu) =\inf_{\pi}c(\pi)$. 
\begin{proposition}
\label{prop_deltas}Let $X,Z$ be independent random vectors on $B \subset R^d$ and $Y$ a random vector on $A \subset \R^n$. We assume that $Y$ is not constant i.e. $P_Y\neq \delta_a$ for any $a\in A$. Let $\xi_i=(y_i,x_i,z_i):\Omega\to A\times B^2$ for $i=1,\ldots,l$ be independent random vectors distributed as $P_{Y,X,Z}$. Let $\mu(\omega)=\frac 1 k \sum_{i=1}^k\delta_{y_i(\omega),x_i(\omega)}$, $\nu(\omega)=\frac 1 k \sum_{i=1}^k\delta_{y_i(\omega),z_i(\omega)}$ and let $\pi_{\Delta}$ be the diagonal coupling and let $\pi\neq\pi_{\Delta}$ be any other coupling. Define the random variables $\tilde{c}(\alpha):\omega\mapsto c(\alpha,\mu(\omega),\nu(\omega))$ for any coupling $\alpha$. Then
\[
\E_{\omega\in\Omega}[\tilde{c}(\pi_{\Delta}]<\E_{\omega\in \Omega}[\tilde{c}(\pi)].
\]
\end{proposition}

Now we will numerically verify this in the case of image denoising on MNIST \cite{deng2012mnist}. We draw $100$ random samples from the MNIST dataset $(x_i)_{i=1}^{100}$ and add noise with standard deviation $0.1$ on them to create samples from the joint distribution $(Y,X)$. Then we randomly sample 100 $(z_i)_{i=1}^{100}$ from the uniform distribution on $28 \times 28$. We use the python optimal transport package (POT) \cite{flamary2021pot} to estimate the optimal coupling between the discrete samples $(y_i,x_i)$ and $(y_i,z_i)$ which can be seen in Fig. \ref{fig_coup}. Thus in this example it is reasonable to assume that with random draws of $z,y$ and $x$ one should expect the optimal plan to be diagonal, even with respect to $W_1$.

\begin{figure}
    \centering
    \includegraphics[scale = 0.6]{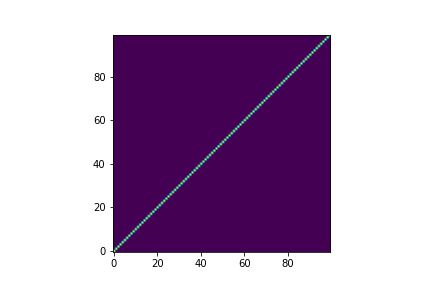}
    \caption{Optimal coupling plan in the MNIST example.}
    \label{fig_coup}
\end{figure}
\begin{remark}\label{comp:cond} In the special case where $P_{Y,X}= \sum_{i=1}^n\delta_{(y_i,x_i)}, P_{Y,Z}=\sum_{i=1}^n\delta_{(y_i,z_i)}$ with $y_i\neq y_j$ for $i\neq j$ the computation of the conditional Wasserstein distance is reduced to the computation of the distances $\Vert x_i-z_i\Vert$ i.e. $W^2_{2,Y}(P_{Y,X},P_{Y,Z})=\frac 1 n\sum_i \Vert x_i-z_i\Vert ^2$
\end{remark}
Lastly, in order to obtain a cost for which the optimal transport plan $\alpha$ almost fulfills $\pi^{1,3}_{\sharp}\alpha=\Delta_{\sharp}P_Y$ we define the metric $d_\beta((y_1,x_1),(y_2,x_2)) = \Vert x_1 - x_2 \Vert + \beta \Vert y_1-y_2 \Vert.$ For big values of $\beta$ it thus is very costly to move mass in $y$-direction.

\begin{proposition}\label{prop_beta}
Consider the Wasserstein metric $W_{p,d_\beta}$ with respect to the metric $d_\beta$. Then for $\beta \rightarrow \infty$ we have that every sequence of optimal transport plans $\alpha^{\beta}$ with respect to $W_{p,d_\beta}$ between measures $P_{Y,X}$ and $P_{Y,Z}$ has diagonal cost going to zero, i.e.,

\[\int_{A^2} \|y_1-y_2\|^p d{\pi^{1,3}}_{\#}(\alpha^{\beta}) \rightarrow 0.\]

\end{proposition}

\section{Conditional Sinkhorn Generators}
We now outline a simple idea in order to leverage our conditional Wasserstein distance to train a conditional generator for solving Bayesian inverse problems. We base our algorithms on the idea of GANs \cite{gan_paper}, conditional GANs \cite{mirza2014conditional} and Sinkhorn generative models \cite{genevay18}. In particular, in \cite{genevay18} it is outlined how to train a generative model via the Sinkhorn divergence $S_{\varepsilon}$, which interpolates between Wasserstein and MMD. Since the Sinkhorn divergence approximates the Wasserstein distance for small blurs $\varepsilon$ we can use it to efficiently learn conditional generators. 

We benchmark the following algorithms: 
\begin{itemize}
    \item Joint Sinkhorn generator: Train via minimizing $S_{\varepsilon}(P_{Y,X},P_{Y,G(Y,Z)})$ for a generator $G$ with latent $P_Z \sim \mathcal{N}(0,I)$. 
\item $\beta$-posterior Sinkhorn generator: In light of Proposition \ref{prop_beta}, we can also consider the Sinkhorn divergence $\tilde{S}_{\varepsilon}$ with respect to the cost $d((y_1,x_1),(y_2,x_2)) = \Vert x_1 - x_2 \Vert^2 + \beta \Vert y_1-y_2 \Vert^2$ for large $\beta$. 
\item Diagonal generator: We draw first $x \sim P_x$, then calculate $y$ according to the forward model and assume that we obtain distinct samples for $y$. Thus by Remark \ref{comp:cond} our conditional Wasserstein distance $W_{2,Y}$ is minimized by just matching $z_i$ and $x_i$ via a simple MSE loss. 
\end{itemize}  
Runtime wise a training run for the Random Images experiment takes about 7 seconds for the diagonal GAN as this can be trained using a MSE. The joint flow trains in about 56 seconds and the $\beta$-Sinkhorn is slowest with 103 seconds on an NVIDIA GeForce RTX 2060.

Note that the aim of the next experiments is to underline our theory and not to outperform other algorithms such as conditional normalizing flows on these problems. 
\subsection{Mixture Models}
First we will check its ability to estimate the posteriors in a Bayesian inverse problem considered in \cite{HHS22} with analytically known posteriors. We base our code on \cite{HHS22} and refer to the details in \cite{HHS22}, but essentially it is a linear inverse problem $Y=AX + \eta$ with Gaussian likelihood $\eta$ and $X \sim P_X$ distributed according to a Gaussian Mixture Model. Then in particular the posteriors $P_{X|Y=y}$ are Gaussian Mixture Models and therefore we are able to evaluate the Sinkhorn divergence. Furthermore, $\varepsilon$ is chosen small enough such that it holds $S_{\varepsilon} \approx \frac{1}{2}W_2$.

\begin{figure}
    \begin{subfigure}[b]{0.32\textwidth}
    \includegraphics[width=4.5cm]{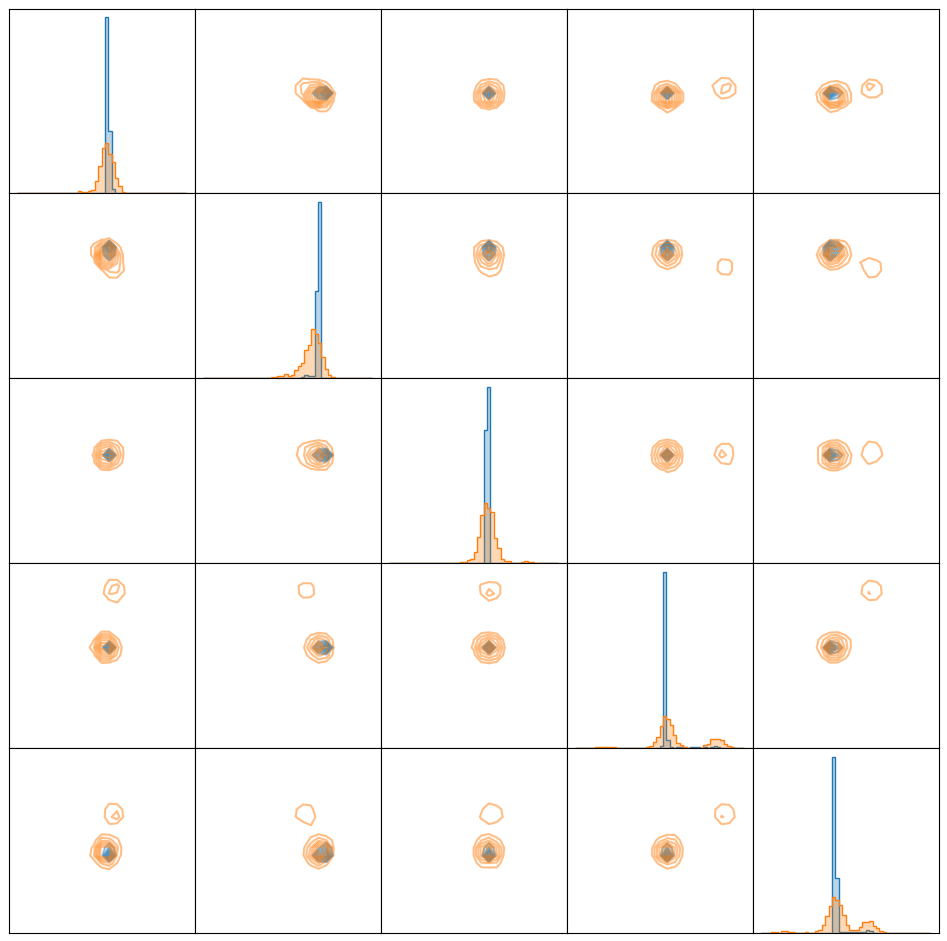}
    \caption{$\beta$-Sinkhorn generator.}
    \end{subfigure}
    \begin{subfigure}[b]{0.32\textwidth}
    \includegraphics[width=4.5cm]{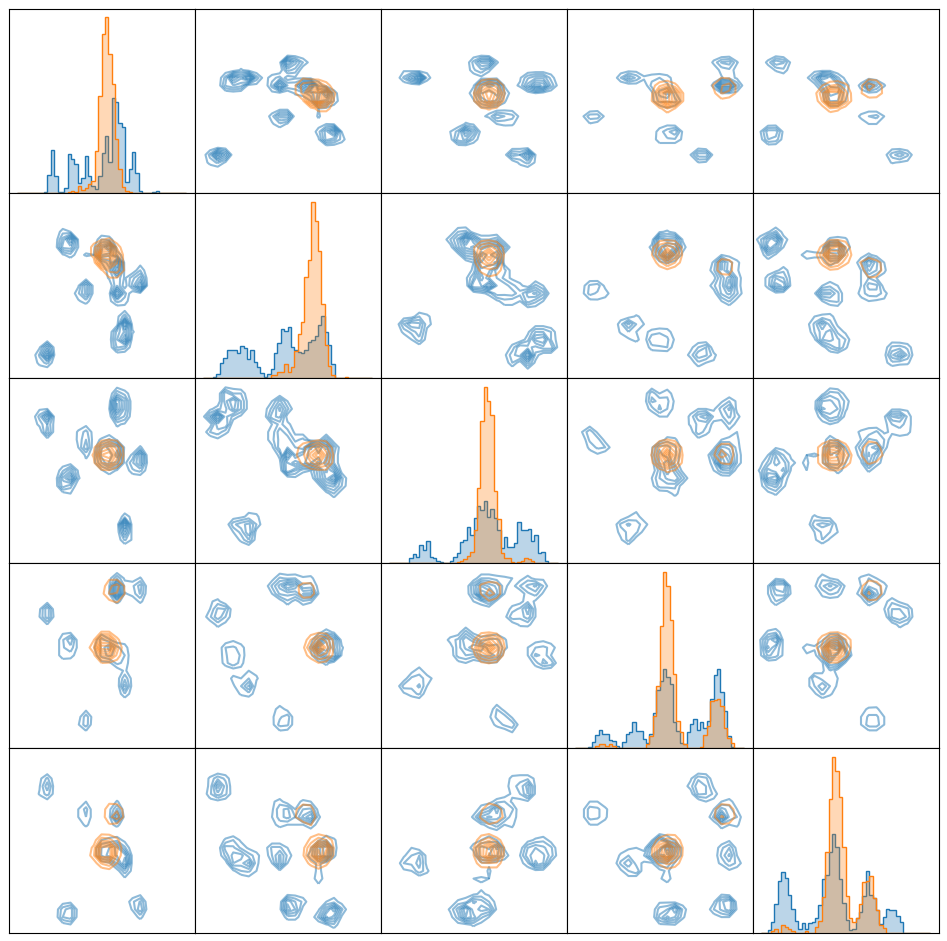}
    \caption{Joint Sinkhorn.}
    \end{subfigure}
    \begin{subfigure}[b]{0.32\textwidth}
    \includegraphics[width=4.5cm]{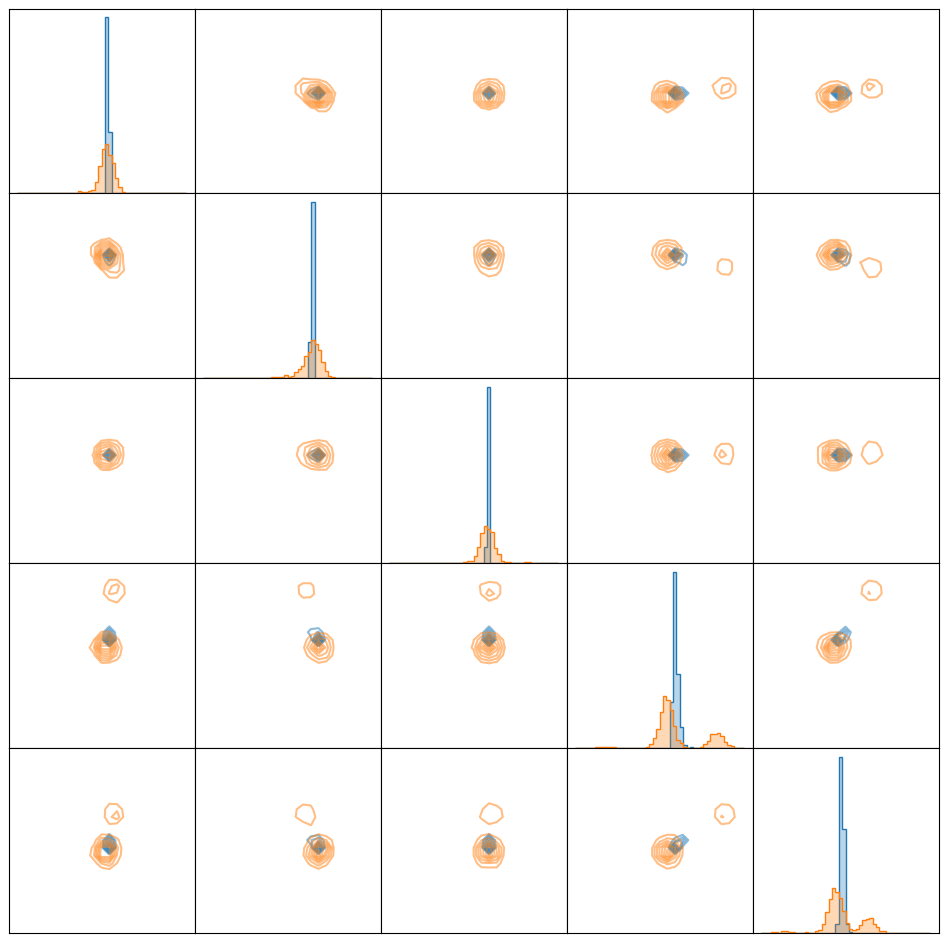}
    \caption{Diagonal Sinkhorn.}
    \end{subfigure}
    \caption{Histograms for the two conditional Sinkhorn generators.}
    \label{fig_mix}
\end{figure}

In particular the forward operator is chosen diagonally with $A_{i,i} = \frac{0.2}{i+1}$, $\eta \sim \mathcal{N}(0,0.03^2)$ and the $P_X$ is chosen with 8 modes, random means in $[-1,1]^5$ and variance of $0.01\ I$. We train three conditional generators $G$ are standard feedforward neural networks minimizing the different losses for 5000 iterations using the Adam optimizer \cite{kingma2014adam} with a batch size of 1024. The differentiable calculation of the Sinkhorn divergences is done via the GeomLoss python package \cite{feydy2019interpolating}. We average the quantity $\mathbb{E}_y\left[S_{\varepsilon}(P_{X|Y=y}, G(y,\cdot)_{\#}P_Z)\right]$ over 100 $y$ where each measure is approximated using 1000 samples, which we refer to as the expected posterior error. Furthermore we evaluate the joint Sinkhorn divergence error for $5000$ $(y,x)$ pairs with the generated counterparts $(y,G_y(Z))$.

\begin{table}[t]
\begin{center}
\begin{tabular}[t]{c|ccc} 
              & joint   & Baseline & $\beta$-Sinkhorn \\
\hline
Joint Error  & 0.022  & 0.034   &  0.020   \\      

Expected Posterior         & 0.762  & 0.089  &  0.087   \\ 
\end{tabular}
\caption{Comparison of different methods in the Mixture example.}
\label{tab:mix}
\end{center}
\end{table}

One can see the results in table \ref{tab:mix}, where we compare the baseline (diagonal) method with the joint Sinkhorn method and the $\beta$-Sinkhorn. As expected from the theory, the diagonal and the $\beta$-Sinkhorn perform well for the expected posterior error. For the joint distribution the diagonal performs much worse than the other ones. This can be explained by the conditioning collapse phenomenon \cite{shahbazi2022collapse}. 

We depict three exemplary posterior histogram plots in Figure \ref{fig_mix}, where one can see that the naive joint Sinkhorn generator predicts something close to the prior whereas the $\beta$-Sinkhorn does a better job approximating the posterior distribution. 

\subsection{Random Images}
In this experiment, we choose the prior distribution $X \sim U([0,1]^d)$ where $d = 16$ is the dimension. We set $Y = \sum_{i=1}^d \frac{X_i}{d}$. The inverse problem consists in sampling $4\times4$-images from the observation $y \in \mathbb{R}$ such that the images $x$ given $y$ have sum $dy$. This problem was suggested in \cite{braitinger}. We train all algorithms with batch size 1024 for 5000 iterations using Adam \cite{kingma2014adam} with learning rate 1e-4.

Since we know that all the generated samples $x$ by the generator $G$ for fixed $y$ should have $\sum_{i=1}^d \frac{x_i}{d} = y$, we propose to evaluate the following two metrics. First we sample random $Y \sim P_Y$ and $Z \sim P_Z$ according to the latent distribution. Then we compare the same joint error $S_{\varepsilon}(P_{Y,X}, P_{Y,G(Y,\cdot)_{\#}P_Z})$ for the three generators $G$. Secondly, we evaluate how close the predicted samples are to $y$ therefore we also test on the "resimulation" error \cite{kruse2021benchmarking}, i.e., we test on $E_y[\frac{1}{K}\sum_{i=1}^K \Vert G(y,z_k)-y\Vert^2]$ for randomly drawn latents $z_k \sim P_Z$ and the different models $M$. Furthermore, the expectation is approximated via $100$ draws from $P_Y$. This measures how well the posterior is adapted to the forward model. 

\begin{table}[t]
\begin{center}
\begin{tabular}[t]{c|ccc} 
              & joint   & Baseline & $\beta$-Sinkhorn \\
\hline
Joint Error         & 0.279  & 0.610  &  0.285   \\ 
Average Resim.   & 0.004  & 0.0001   &  0.0005   \\    
\end{tabular}
\caption{Comparison of different methods in the random images example.}
\label{tab:img}
\end{center}
\end{table}

We train the three generator networks of the same size for 5000 iterations. We evaluate the two metrics and average this over 10 training runs. The joint Sinkhorn divergence is approximated using $5000$ samples, the resimulation error is approximation using $K = 2000$ latent samples for each $y$. The results can be seen in Table~\ref{tab:img}, where we can infer that the $\beta$-Sinkhorn generator gives a nice tradeoff between a good joint Sinkhorn divergence as well as a good resimulation error. Note that this also nicely aligns with our theory, as the baseline can be obtained by taking $\beta \rightarrow \infty$ and is therefore optimizing the sharpest posterior bound. Generated samples of the method can be found in the appendix. 
\section{Conclusions}
In this paper we introduced the conditional Wasserstein distance, inspired from applications in inverse problems. We are able to rewrite this as an expectation with respect to the observation and are therefore able to directly infer posterior guarantees when trained on those. Furthermore, we calculated its dual when the probability measures are compactly supported and recovered well-known conditional Wasserstein GAN losses. Furthermore, we outlined some sufficient conditions under which this conditional Wasserstein distance equals the usual one. However, finding necessary and sufficient conditions under which couplings are contained in our restricted set, is an open question. One way to approach this could be by using equivalent PDE formulations for finding the dual such as done in \cite{asokan2023data, mroueh2018sobolev}.
\bibliographystyle{abbrv}
\bibliography{references}

 \appendix
\section{Proof of Proposition \ref{prop:3plans}}
    \begin{proof} 
    In order to make notation easier we permute the factors and view $\alpha\in \Gamma_Y$ as measure on $A^2\times B^2$ and thus we have to show that $\pi^{2,3,4}_{\sharp}$ is a bijection. We will show that $(\Delta\circ \pi^2,\pi^3,\pi^4 )_{\sharp}$ is the inverse of $\pi^{2,3,4}_{\sharp}$. Since $id_{(A\times B^2)} = \pi^{2,3,4}\circ(\Delta\circ\pi^2,\pi^3,\pi^4)$ we are left to show that $( \Delta\circ\pi^2, \pi^3,\pi^4)_{\sharp}\circ\pi^{2,3,4}_{\sharp}=id_{\Gamma_Y(P_{Y,X},P_{Y,Z})}$ which follows from 
    \begin{align*}
    &\int_{A^2\times B^2}f(y_1,x_1,y_2,x_2)\d ( \Delta\circ \pi^2,\pi^3,\pi^4)_{\sharp}\pi^{2,3,4}_{\sharp}\alpha \\
    &= \int_{A^2\times B^2}f(y_2,x_1,y_2,x_2)\d \alpha\\
    &= \int_{A^2\times B^2}f(y_2,x_1,y_2,x_2)\d \gamma_{y_1,y_2}\d \pi^{1,3}_{\sharp}\alpha\\
    &= \int_{B^2\times A}f(y,x_1,y,x_2)\d \gamma_{y,y}(x_1,x_2)\d P_Y(y)\\
    &= \int_{B^2\times A^2}f(y_1,x_1,y_2,x_2)\d \gamma_{y_1,y_2}(x_1,x_2)\d \Delta_{\sharp}P_Y(y_1,y_2)\\
    &= \int_{A^2\times B^2}f\d\alpha
    \end{align*}
    for all $\alpha\in\Gamma_Y(P_{Y,X},P_{Y,Z})$ and all measurable functions $f:A^2\times B^2\to [0,\infty)$.
    
    In order to show the second claim we note that $\gamma_{y,y}$ is the disintegration of $\pi^{2,3,4}_{\sharp}\alpha$ if $\gamma_{y_1,y_2}$ is a disintegration of $\alpha\in\Gamma_Y(P_{Y,X},P_{Y,Z})$ w.r. $\pi^{1,3}$ . This follows from
        \begin{align*}
        \int f(y,x_1,x_2)\d\gamma_{y,y}\d P_Y & = \int f(y_2,x_1,x_2)\d\gamma_{y_1,y_2}\d \Delta_{\sharp}P_Y\\
        &= \int f(y_2,x_1,x_2) \d \alpha \\&= \int f(y_2,x_1,x_2)\d \pi^{2,3,4}_{\sharp}\alpha.
        \end{align*}
        Thus it follows from \eqref{mod:was} that 
        \begin{align*}
        &\int_{(B\times A)^2}\|(x_1,y_1)-(x_2,y_2)\|^p\d \alpha \\
        &= \int_A\int_{B^2} \|x_1-x_2\|^p\d \pi^{2,3,4}_{\sharp}\alpha
        \end{align*}
\end{proof}
\section{Exchanging supremum and infimum}
\label{app:ex}
The proof of strong duality relies on the following Minimax principle, from \cite[Theorem 7 Chapter 6]{aubin2006applied}.
\begin{theorem} \label{minmax}
Let $X$ be a convex subset of a topological vector space, and $Y$ be a convex subset of a vector space. Assume $f: X \times Y \to \R$ satisfies the following conditions:
\begin{enumerate}
    \item For every $y\in Y$ the map $x\to f(x,y)$ is lower semi continuous and convex.
    \item There exists $y_0$ such that $x\to f(x,y_0)$ is inf-compact i.e the set $\{x\in X: f(x,y_0) \leq a\}$ is relatively compact for each $a \in \R$.
    \item For every $x\in X$ the map $y \to f(x,y)$ is convex.
\end{enumerate}
Then we have, 
\begin{align*}
    \inf_{x\in X} \sup_{y\in Y} f(x,y) = \sup_{y\in Y} \inf_{x\in X} f(x,y)
\end{align*}
\end{theorem} 

\begin{theorem}\label{thm:change}Let $A \subset \R^d, B \subset \R^n$ be compact, denote the Lagrangian by
\begin{align*}
    &L(\alpha,f,g):= \mathbb{E}_{P_{X,Y}}[f] + \mathbb{E}_{P_{Z,Y}}[g] + \\
    & \int_{(A\times B)^2} \|(y_1,x_1)-(y_2,x_2)\| -f(y_1,x_1) -g(y_2,x_2)  \d \alpha.
\end{align*}
Then it holds that
\begin{align*}
      \underset{\alpha\in M}{\mathsf{inf}} \sup_{f,g \in C_b} L(\alpha,f,g) = \sup_{f,g \in C_b}\underset{\alpha\in M}{\mathsf{inf}} L(\alpha,f,g) 
\end{align*}

 i.e
 \begin{align} \label{PrimDual}
     &\underset{\alpha\in \Gamma_Y(P_{Y,X}, P_{Y,Z})}{\mathsf{inf}}\int_{(A\times B)^2} \|(y_1,x_1)-(y_2,x_2)\|\d \alpha = \sup_{(f,g) \in D} \mathbb{E}_{P_{Y,X}}[f] + \mathbb{E}_{P_{Y,Z}}[g]
 \end{align}
\end{theorem}
\begin{proof} 
We will verify the conditions in Theorem \ref{minmax}
Recall that $M$ is the set of finite nonnegative Borel measures $\alpha$ on $(A\times B)^2$ s.t. there exists a finite nonegative finite measure $\beta$ on $B$ with $\pi^{1,3}_{\sharp}\alpha = (\Delta_B)_{\sharp} \beta$. Let $\mathcal{M}$ be the topological vector space of finite signed Borel measures on $(A\times B)^2$ with weak convergence topology. Thus since the pushforward is linear on $\mathcal{M}$ we conclude that $M$ is a convex subset. $M$ will serve as the set $X$ in Theorem \ref{minmax}, $C_b\times C_b$ will serve as $Y$ and $L$ will serve as $f$. \\
\textit{Verifying 1.} The map $\alpha \to L(\alpha,f,g)$ is linear and continuous on $M$ ($\Pi(P_{Y,X}, P_{Y,Z})$) under the weak convergence of measures. This follows from the fact that the integrand of $\alpha$ in $L(\alpha, f,g)$ is in $C_b((A\times B)^2)$. Hence we verified $1.$ of Theorem \ref{minmax}.\\
\textit{Verifying 3.} Note that for any $\alpha \in M$ the map $(f,g) \to L(\alpha,f,g)$ is linear in $(f,g)$ and therefore convex.  \\
\textit{Verifying 2.} Setting $f(y,x) =-1, g(y,x)=-1$ for all $(y,x)$, we will show the set 
\[M_a:= \{\alpha \in  M: L(\alpha,-1,-1) \leq a\}\]
is inf-compact.
Since the integrand is bounded from below by $2$ and $M$ only contains nonnegative measures, it is clear that the measures in $M_a$ are uniformly bounded in the total variation norm. Otherwise we would obtain that 
\begin{align*}
    &\int \left(\|(y_1,x_1)-(y_2,x_2)\| -f(y_1,x_1) -g(y_2,x_2)\right)  \d \alpha\\
    &= \int \left(\|(y_1,x_1)-(y_2,x_2)\|+2\right)  \d \alpha \rightarrow \infty,
    \end{align*}
    which contradicts  $L(\alpha,-1,-1) \leq a$.  Therefore the compactness of $A,B$ implies that  $M_a$ is a family of tight measures. By \cite[Theorem 8.6.7]{Bogachev2007}, the set $M_a$ is relatively compact in the weak topology. 
    Using Theorem \ref{minmax} we can conclude the proof.
\end{proof}

\section{Proof of Proposition ~\ref{prop:ind}}
\begin{proof}
Let $\alpha\in \Gamma(P_{Y,X},P_{Y,Z})$. Then
\begin{align*}
\int\|(y_1,x_1)-(y_2,x_2)\|^p\d \alpha &\geq \int \|x_1-x_2\|^p\d \alpha \\
&=\int \|x_1-x_2\|^p\d\pi^{2,4}_{\sharp}\alpha\\
&\geq \left(W_p(P_X,P_Y)\right)^p
\end{align*}
and hence $W_p(P_{Y,X},P_{Y,Z})\geq W_p(P_X,P_Z)$. Let now $\gamma \in \Gamma(P_X,P_Z)$ be an optimal coupling. Since $P_{Y,X}=P_Y\times P_X$ and $P_{Y,Z}=P_Y\times P_Z$ we have that modulo permutation of factors $\gamma\times \Delta_Y\in \Gamma_{P_Y}(P_{Y,X},P_{Y,Z})$. Thus 
\begin{align*}
W_{p,Y}(P_{Y,X},P_{Y,Z})^p &\leq \int \|(x_1,y_1)-(x_2,y_2)\|^p\d \gamma\times\Delta_Y\\
&= \int \|(x_1,y)-(x_2,y)\|^p \d P_Y\d \gamma\\
&= \int \|x_1-x_2\|^p\d\gamma \\
&= W_p(P_X,P_Z)^p.
\end{align*}
Since $W_{p,Y}(P_{Y,X},P_{Y,Z})\geq W_p(P_{Y,X},P_{Y,Z})$ we obtain equality everywhere and $\gamma\times\Delta_Y$ is an optimal coupling for $W_p$.
\end{proof}
\section{Proof of Proposition \ref{prop_deltas}}

\begin{proof}
We have for a plan $\alpha$ that
\begin{align*}
        \E[\tilde{c}(\alpha)] &= \E_{\prod_iP_{\xi_i}}\left[\sum_{i,j}\Vert (y_i,x_i)-(y_{j},z_j)\Vert^2 \alpha_{i,j}\right]\\
&=\sum_{i,j}\alpha_{i,j}\E_{P_{x_i}\times P_{z_j}}\left[\Vert x_i-z_j\Vert^2\right]
+\sum_{i,j} \alpha_{i,j}\E_{P_{y_i}\times P_{y_j}}\left[\Vert y_i-y_j\Vert^2\right]\\
&= \sum_{i,j}\alpha_{i,j}\E_{P_X\times P_Z}\left[\Vert x-z\Vert^2\right]+\sum_{i\neq j}\alpha_{i,j}\E_{P_Y\times P_Y}\left[\Vert y-\bar{y}\Vert^2\right]\\
&= \E\left[\Vert x-z\Vert^2\right] + \sum_{i\neq j}\alpha_{i,j}\E_{P_Y\times P_Y}\left[\Vert y-\bar{y}\Vert^2\right]
\end{align*}
Since $P_Y$ is not supported on only a single point we can chose sets $U_1,U_2\subset A$ such that $P_Y(U_i)> \epsilon>0$ and $\delta := \inf\{\Vert u_1-u_2\Vert^2:u_1\in U_1, u_2\in U_2\}>0$. Then $P_Y\times P_Y(U_1\times U_2)>\epsilon^2$ and $\Vert y-\bar{y}\Vert\geq \delta$ on $U_1\times U_2$. Thus
\begin{align*}
        \E_{P_Y\times P_Y}\left[\Vert y-\bar{y}\Vert^2\right]\geq \epsilon^2\delta>0
\end{align*}
and consequently since there exists $i\neq j$ with $\pi_{i,j}>0$ we can conclude that
\begin{align*}
\E[\tilde(c)(\pi)]-\E[\tilde(c)(\pi_{\Delta})]&= \sum_{i\neq j}\pi_{i,j}\E_{P_Y\times P_Y}[\Vert y-\bar{y}\Vert^2].
\end{align*}
\end{proof}

\section{Proof of Proposition~\ref{prop_beta} and Verification}
\begin{proof}
Denote by $\pi_{opt}$ the optimal transport plan associated to the conditional Wasserstein metric $W_{p,Y}$. 
This plan exists by Corollary~\ref{cor_plan} with finite cost $d_{\pi_{opt}}$. Furthermore $d=d_{\beta}$ a.e. for $\pi_{opt}$.  Then for an optimal plan $\alpha^\beta$ for $W_{p,d_\beta}$ we have that
\begin{align*}
    d_{\pi_{opt}}&=\int_{(A\times B)^2}d_{\beta}((y_1,x_1),(y_2,x_2)) d\pi_{opt} \\
                &\geq \int_{(A\times B)^2}d_{\beta}((y_1,x_1),(y_2,x_2))d\alpha^\beta\\
                &\geq\int_{B^2}\|x_1-x_2\|^pd\pi^{2,4}_{\sharp}\alpha^\beta\\
                &+ \beta\int_{A^2}\|y_1-y_2\|^pd\pi^{1,3}\alpha^\beta\\
                &\geq \beta\int_{A^2}\|y_1-y_2\|^pd\pi^{1,3}\alpha^\beta
\end{align*}
and thus the claim.

\end{proof}

To verify this convergence, we tested it on a numerical example. Here we sampled $X \sim \mathcal{N}(0,I_2)$ and $Z \sim U([0,1]^2)$. We calculated $Y = X + 0.02\ Z + 0.05\ \eta$, where $\eta \sim \mathcal{N}(0,I_2).$ Then we simulated the optimal transport plans $\pi^{\beta}$ for $\beta \in \{1,10,100,1000,10000,100000\}$ and counted the sum of nondiagonal elements, which gives the fraction of the total mass located on the nondiagonal. A plot (with a logarithmic x-axis) is given in Fig. \ref{fig_beta}. 

\begin{figure}
    \centering
    \includegraphics[scale = 0.6]{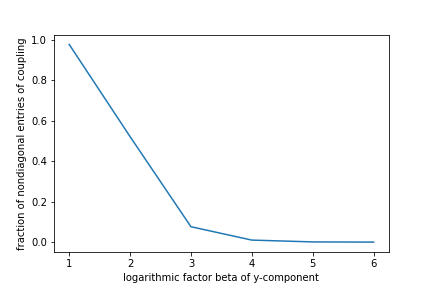}
    \caption{Verification of convergence of the transport plan to a diagonal one as $\beta \rightarrow \infty$.}
    \label{fig_beta}
\end{figure}
\section{Samples for random Images experiment}
In this section one can see generated samples. One can see that in all cases there is not any very obvious mode collapse going on, although judging from the numbers, the baseline Sinkhorn should have some diversity issues going on, see Fig.~\ref{fig_img}.
\begin{figure}
    \begin{subfigure}[b]{0.32\textwidth}
    \includegraphics[width=5cm]{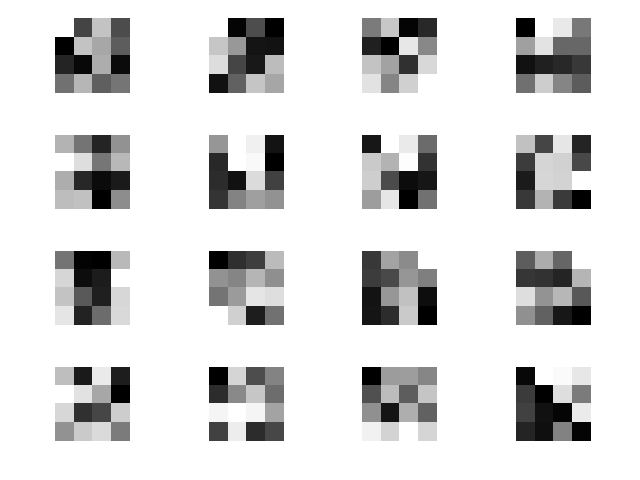}
    \caption{$\beta$-Sinkhorn generator.}
    \end{subfigure}
    \begin{subfigure}[b]{0.32\textwidth}
    \includegraphics[width=5cm]{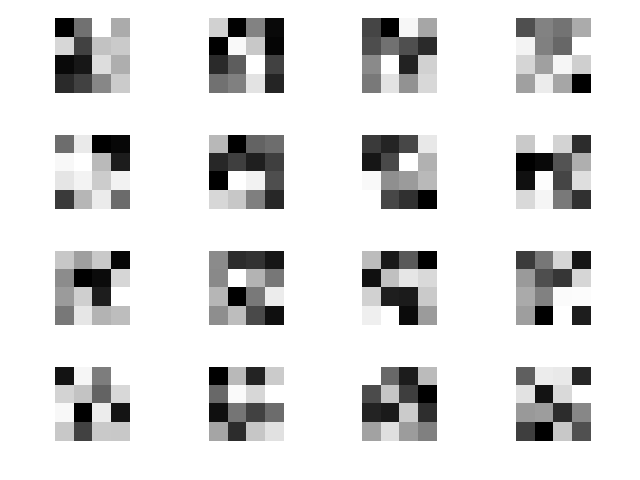}
    \caption{Joint Sinkhorn.}
    \end{subfigure}
    \begin{subfigure}[b]{0.32\textwidth}
    \includegraphics[width=5cm]{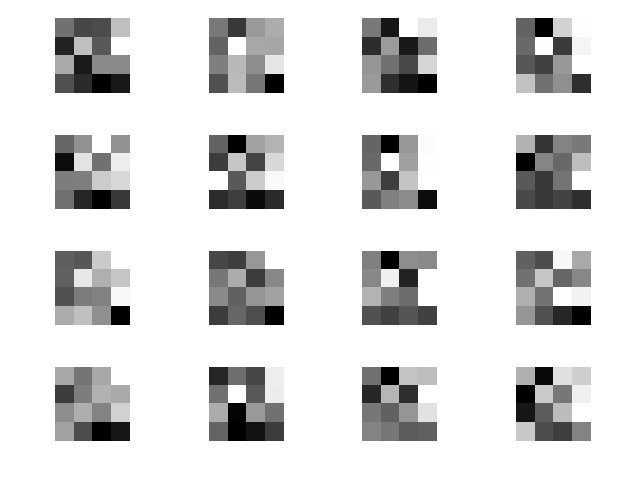}
    \caption{Diagonal Sinkhorn.}
    \end{subfigure}
    \caption{Histograms for the three conditional Sinkhorn generators.}
    \label{fig_img}
\end{figure}
\section{Standard deviation for the numerical examples}
In this section we provide also the standard deviations for the two numerical examples. 

For the image example we provide the standard deviations in table \ref{tab:mix_std}. The standard deviations are with respect to the 10 training runs of the different methods, showing that the differences between the runs are quite small. Only the conclusion that $\beta$-Sinkhorn is superior to baseline in terms of expected posterior error is not supported, as the standard deviation are bigger than the differences.

\begin{table}[t]
\begin{center}
\begin{tabular}[t]{c|ccc} 
              & joint   & Baseline & $\beta$-Sinkhorn \\
\hline
Joint Std         & 0.0026  & 0.0005  &  0.0015   \\ 
Exp Std.   & 0.042  & 0.007  &  0.01   \\    
\end{tabular}
\caption{Comparison of standard deviations of different methods in the mixture example.}
\label{tab:mix_std}
\end{center}
\end{table}

For the image example we provide the standard deviations in table \ref{tab:img_std}. We can see that the standard deviations are much smaller than the differences between the methods. 

\begin{table}[t]
\begin{center}
\begin{tabular}[h!]{c|ccc} 
              & joint   & Baseline & $\beta$-Sinkhorn \\
\hline
Joint Std         & 0.0009  & 0.0019  &  0.001   \\ 
Resim Std.   & 0.00025  & 8.53e-6   &  1.76e-5   \\    
\end{tabular}
\caption{Comparison of standard deviations of different methods in the random images example.}
\label{tab:img_std}
\end{center}
\end{table}

\end{document}